\documentclass{article}

\usepackage[preprint]{nips_2018}

\usepackage[utf8]{inputenc} 
\usepackage[T1]{fontenc}    
\usepackage{hyperref}       
\usepackage{url}            
\usepackage{booktabs}       
\usepackage{amsfonts}       
\usepackage{amsmath}
\usepackage{amssymb}
\usepackage{amsthm}
\usepackage{nicefrac}       
\usepackage{microtype}      

\usepackage{comment}
\usepackage{graphicx} 
\usepackage{subfigure} 
\usepackage{pdflscape}
\usepackage{algorithm}
\usepackage{algorithmic}
\usepackage{bm}
\usepackage{soul}
\usepackage{listings}

\usepackage{multirow}
\usepackage{multicol}

\def\finalversion{0}

\usepackage[framemethod=tikz]{mdframed}
\usepackage{geometry}
\usepackage[many]{tcolorbox}
\usepackage{marginnote}

\newtcolorbox[use counter=mynote]
  {mynote}[1][]
  {title=ToDo~\thetcbcounter,
   width=1.3in, 
   left=0pt,
   right=0pt,
   fonttitle=\bfseries\color{black},
   colframe=pink,
   colback=pink!10,
   #1
}

\newcommand{\solicitation}[1]{\if\finalversion1 {} \else \begin{mdframed}[hidealllines=true,backgroundcolor=blue!20]From Solicitation: \textit{#1}\end{mdframed} \fi}

\newtheorem{theorem}{Theorem}[section]
\newtheorem{lemma}[theorem]{Lemma}

\DeclareMathOperator{\Div}{div}

\title{A Unified Framework for Diffusion Bridge Problems: Flow Matching and Schr\"{o}dinger Matching into One
}

\author{%
  Minyoung Kim
  \\
  Samsung AI Center Cambridge, UK \\
  \texttt{mikim21@gmail.com} \\
}

\begin{document}

\maketitle

\begin{abstract}
The {\em bridge problem} is to find an SDE (or sometimes an ODE) that bridges two given distributions. The application areas of the bridge problem are enormous, among which the recent generative modeling (e.g., conditional or unconditional image generation) is the most popular. Also the famous Schr\"{o}dinger bridge problem, a widely known problem for a century, is a special instance of the bridge problem. Two most popular algorithms to tackle the bridge problems in the deep learning era are: {\em (conditional) flow matching} and {\em iterative fitting} algorithms, where the former confined to ODE solutions, and the latter specifically for the Schr\"{o}dinger bridge problem. The main contribution of this article is in two folds: i) \textbf{We provide concise reviews of these algorithms with technical details to some extent}; ii) \textbf{We propose a novel unified perspective and framework that subsumes these seemingly unrelated algorithms (and their variants) into one}.  In particular, we show that our unified framework can instantiate the Flow Matching (FM) algorithm, the (mini-batch) optimal transport FM algorithm, the (mini-batch) Schr\"{o}dinger bridge FM algorithm, and the deep Schr\"{o}dinger bridge matching (DSBM) algorithm as its special cases. We believe that this unified framework will be useful for viewing the bridge problems in a more general and flexible perspective, and in turn can help researchers and practitioners to develop new bridge algorithms in their fields. 
\end{abstract}


\section{(Diffusion) Bridge Problems}\label{sec:overview}

The \textbf{diffusion bridge problem}, or simply the \textbf{bridge problem}, can be defined as follows.

\textbf{$\bullet$ Bridge problem.}
Given two distributions $\pi_0(\cdot)$ and $\pi_1(\cdot)$ in $\mathbb{R}^d$, find an SDE, more specifically, find the drift function $u_t(x)$ where $u:\mathbb{R}[0,1]\times \mathbb{R}^d \to \mathbb{R}^d$ with a specified diffusion coefficient $\sigma$,
\begin{align}
dx_t = u_t(x_t) dt + \sigma dW_t, \ x_0\!\sim\!\pi_0(\cdot)
\label{eq:bp_sde_fwd}
\end{align}
that yields $x_1\!\sim\!\pi_1(\cdot)$. Here $\{W_t\}_t$ is the Wiener process or the Brownian motion. 
Alternatively one can aim to find a reverse-time SDE (or both). That is, find $\overline{u}_t(x)$ in
\begin{align}
\overleftarrow{d}x_t = \overline{u}_t(x_t) dt + \sigma \overleftarrow{d}W_t, \ x_1\!\sim\!\pi_1(\cdot)
\label{eq:bp_sde_rev}
\end{align}
that yields $x_0\!\sim\!\pi_0(\cdot)$.

Note that if we specify $\sigma\!=\!0$, then our goal is to find an ODE that bridges the two distributions $\pi_0(\cdot)$ and $\pi_1(\cdot)$. 
Once solved, the solution to the bridge problem can give us the ability to sample from one of the $\pi_{\{0,1\}}$ given the samples from the other, simply by integrating the learned SDE.
The application areas of the bridge problem are enormous, among which the generative modeling (e.g., conditional or unconditional image generation) is the most popular. 
For instance, in typical generative modeling, $\pi_0$ is usually a tractable density like Gaussian, while $\pi_1$ is a target distribution that we want to sample from. In the bridge problem, however, $\pi_0$ can also be an arbitrary distribution beyond tractable densities like Gaussians, and we do not make any particular assumption on $\pi_0$ and $\pi_1$ as long as we have samples from the two distributions. 

\textbf{$\bullet$ Two instances of the bridge problem.}
There are two interesting special instances of the bridge problem: the \textbf{Schr\"{o}dinger bridge problem} and the \textbf{ODE bridge problem}. 
\begin{itemize}
\item \textbf{ODE bridge problem.} We strictly restrict ourselves to ODEs, i.e., $\sigma=0$. 
\item \textbf{Schr\"{o}dinger bridge problem.} With $\sigma\!>\!0$, there is an additional constraint that the path measure of the SDE (\ref{eq:bp_sde_fwd}), denoted by $P^u$, is closest to a given reference SDE path measure $P^{ref}$. That is, 
\begin{align}
\min_{u} \ \textrm{KL}(P^u || P^{ref}) \ \ \textrm{s.t.} \ P^u_0(x_0) = \pi_0(x_0), P^u_1(x_1) = \pi_1(x_1)
\label{eq:sb}
\end{align}
where $P_t$ of a path measure $P$ indicates the marginal distribution at time $t$. 
In this case to have finite KL divergence, $\sigma$ of $P^u$ has to be set equal to $\sigma_{ref}$ of $P^{ref}$, i.e., $\sigma = \sigma_{ref}$, (from the Girsanov theorem). 
\end{itemize}

\section{Flow Matching and Schr\"{o}dinger Bridge Matching Algorithms}\label{sec:background}

Among several existing algorithms that aim to solve the bridge problems, in this article we focus on two recent matching algorithms: {\em (conditional) flow matching}~\citep{fm,cfm,stoch_interp,rectified_fm} and {\em Schr\"{o}dinger bridge matching} algorithms~\citep{bortoli21,vargas21,dsbm}. These algorithms were developed independently: the former aimed to solve the ODE bridge problem while the latter the Schr\"{o}dinger bridge problem. 
In this section we review the algorithms focusing mainly on the key ideas with some technical details, but being mathematically less rigorous for better readability. 

In Sec.~\ref{sec:unified}, we propose a novel unified framework that subsumes these two seemingly unrelated algorithms and their variants into one.

\subsection{(Conditional) Flow Matching for ODE Bridge Problems}\label{sec:fm_cfm}

The Flow Matching (FM)~\citep{fm} or its extension Conditional Flow Matching (CFM)~\citep{cfm} is one promising way to solve the ODE bridge problem. 
The key idea of the FM is quite intuitive. We first design some marginal distribution path $\{P_t(x_t)\}_t$ with the boundary conditions $P_0\!=\!\pi_0$, $P_1\!=\!\pi_1$. We then derive the ODE $dx_t = u_t(x_t) dt$ that yields $\{P_t(x_t)\}_t$ as its marginal distributions. The drift $u_t(x_t)$ is approximated by a neural network $v_\theta(t,x_t)$ with parameters $\theta$ by solving:
\begin{align}
\min_\theta \ \mathbb{E}_{t,x_t\sim P_t} || u_t(x_t) - v_\theta(t,x_t) ||^2
\label{eq:fm_optim}
\end{align}
Once solved, we can generate samples from $\pi_1$ (or $\pi_0$) approximately by simulating: $dx_t = v_\theta(t,x_t) dt$, $x_0\!\sim\!\pi_0$ (resp., $x_1\!\sim\!\pi_1$).
However, one of the main limitations of this strategy is that designing the marginal path $\{P_t(x_t)\}_t$ satisfying the boundary condition is often difficult. And it is this issue that motivated the CFM.

\textbf{$\bullet$ Conditional Flow Matching (CFM).}
To make the path design easier, we introduce some latent random variable $z$ to condition $x_t$. Although CFM derivations hold regardless of the choice of $z$, it is typically chosen as the terminal random variates $z=(x_0,x_1)$, 
and we will follow this practice and notation.
Specifically, in CFM we design the so-called {\em pinned marginal path} $\{P(x_t|x_0,x_1)\}_t$ and the {\em coupling distribution} $Q(x_0,x_1)$ subject to the condition $P_0(x_0)=\pi_0(x_0)$, $P_1(x_1)=\pi_1(x_1)$ where $P_t(x_t)$ is defined as
\begin{align}
P_t(x_t) := \int P_t(x_t|x_0,x_1) Q(x_0,x_1) d(x_0,x_1)
\label{eq:cfm_marginal}
\end{align}
Then we derive the ODE $dx_t = u_t(x_t|x_0,x_1) dt$ that yields $\{P(x_t|x_0,x_1)\}_t$ as its marginal distributions for each $(x_0,x_1)$, which admits a closed form if $\{P(x_t|x_0,x_1)\}_t$ are Gaussians~\citep{fm}. We then approximate $\mathbb{E}[u_t(x_t|x_0,x_1)|x_t]$, the conditional expectation derived from the joint $P(x_t|x_0,x_1) Q(x_0,x_1)$, by a neural network $v_\theta(t,x_t)$ by solving the following optimization:
\begin{align}
\min_\theta \ \mathbb{E} \ || u_t(x_t|x_0,x_1) - v_\theta(t,x_t) ||^2
\label{eq:cfm_optim}
\end{align}
where the expectation is taken with respect to the joint $P(x_t|x_0,x_1) Q(x_0,x_1)$ and uniform $t$. 
Surprisingly, it can be shown~\citep{cfm} that the gradient of the objective in (\ref{eq:cfm_optim}) coincides with that in (\ref{eq:fm_optim}) for $u_t(x_t)$ defined as:
\begin{align}
u_t(x) = \frac{1}{P_t(x_t)} \mathbb{E}_{Q(x_0,x_1)}[u_t(x_t|x_0,x_1)P_t(x_t|x_0,x_1)]
\end{align}
hence sharing the same training dynamics as (\ref{eq:fm_optim}). Therefore the optimal $v_\theta(t,x_t)$ of (\ref{eq:cfm_optim}) is a good estimate for $u_t(x_t)$. Since the ODE $dx_t = u_t(x_t) dt$ admits $\{P_t(x_t)\}_t$ as its marginal distributions, so does $dx_t = v_\theta(t,x_t) dt$ approximately. 

Many existing flow matching variants including FM~\citep{fm}, Stochastic Interpolation~\citep{stoch_interp}, and Rectified FM~\citep{rectified_fm} can be viewed as special instances of this CFM framework. For instance, these models can be realized by having a straight line (linear interpolation) pinned marginal path (\ref{eq:linear_pinned}) with vanishing variance while one boundary (e.g., $\pi_0$) is fixed as standard normal $\mathcal{N}(0,I)$.

\textbf{$\bullet$ Limitations of CFM.}
A reasonable choice for the coupling distribution $Q(x_0,x_1)$ is the Optimal Transport (OT) or the entropic OT between $\pi_0$ and $\pi_1$. The pinned marginal $P_t(x_t|x_0,x_1)$ can be chosen as a Gaussian with the linear interpolation between $x_0$ and $x_1$ as its mean, more specifically, 
\begin{align}
P_t(x_t|x_0,x_1) = \mathcal{N}(t x_1 + (1-t) x_0, \beta_t^2 I)
\label{eq:linear_pinned}
\end{align}
for some scheduled variances $\beta_t^2$.
When the combination of the entropic OT $Q(x_0,x_1)$ and $\beta_t = \sigma_{ref} \sqrt{t(1-t)}$ is used, it can be shown that the marginals $\{P_t(x_t)\}_t$ coincide with the marginals of the Schr\"{o}dinger bridge with the Brownian motion reference $P^{ref}: dx_t=\sigma_{ref} dW_t$. However, the main limitations of CFM~\citep{cfm} are: i) CFM solutions are confined to ODEs, hence unable to find the optimal SDE solution to general bridge problems including the Schr\"{o}dinger bridge problem; ii) CFM itself does not provide a recipe about how to solve the entropic OT problem exactly -- what is called SB-CFM proposed in~\citep{cfm} only approximates it with the Sinkhorn-Knopp solution for minibatch data, which is usually substantially different from the population entropic OT solution.

\subsection{Schr\"{o}dinger Bridge Problem}\label{sec:sb}

The Schr\"{o}dinger bridge problem can be defined as (\ref{eq:sb}) where we assume a zero-drift Brownian SDE with diffusion coefficient $\sigma_{ref}$ for the reference path measure $P^{ref}$ throughout the paper. That is,
\begin{align}
P^{ref}: dx_t = \sigma_{ref} dW_t, \ x_0\!\sim\!\pi_0(\cdot)
\label{eq:brownian}
\end{align}
We denote by $P^{SB}$ the Schr\"{o}dinger bridge path measure, i.e., the solution to (\ref{eq:sb}). In the literature, there are two well-known views for $P^{SB}$: the {\em static view} and the {\em optimal control} view. 

The static view has a direct link to the entropic optimal transport (EOT) solution, more specifically
\begin{align}
P^{SB}(\{x_t\}_{t\in[0,1]}) = P^{EOT}(x_0,x_1) \cdot P^{ref}(\{x_t\}_{t\in(0,1)}|x_0,x_1)
\label{eq:sb_static}
\end{align}
where $P^{EOT}(x_0,x_1)$ is the EOT joint distribution solution with the negative entropy regularizing coefficient $2\sigma_{ref}^2$. More formally, 
\begin{align}
P^{EOT}(x_0,x_1) &= \arg\min_{P(x_0,x_1)} \mathbb{E}_{P(x_0,x_1)} ||x_0-x_1||^2 - 2\sigma_{ref}^2 \mathbb{H}(P(x_0,x_1)) \\
& \ \ \ \ \ \ \ \ \ \ \ \ \ \ \ \ \textrm{s.t.} \ \ \ \ P(x_0) = \pi_0(x_0), \ P(x_1) = \pi_1(x_1)
\end{align}
where $\mathbb{H}$ indicates the Shannon entropy. 
We call $P^{ref}(\cdot|x_0,x_1)$ the {\em pinned reference process}, which admits a closed-form Gaussian expression for the specific choice (\ref{eq:brownian}). Although the product form (\ref{eq:sb_static}), i.e., the product of a boundary joint distribution and a pinned path measure, does not in general become Markovian (e.g., It\^{o} SDE representable), the Schr\"{o}dinger bridge is a well-known exception where there exists a unique SDE that yields $P^{SB}$ as its path measure.   

Alternatively, it is not difficult to derive an optimal control formulation for the Schr\"{o}dinger bridge problem. Specifically, $P^{SB}$ can be described by the SDE that has the minimum kinetic energy among those that satisfy the bridge constraint. Letting
\begin{align}
P^{v}: dx_t = v_t(x_t)dt + \sigma_{ref} dW_t, \ x_0\!\sim\!\pi_0(\cdot)
\label{eq:pv}
\end{align}
we have $P^{SB} = P^{v^*}$ where $v^*$ is the minimizer of the following problem:
\begin{align}
\min_{v} \ \mathbb{E}_{P^v}\Bigg[\int_0^1 \frac{1}{2\sigma_{ref}^2} ||v_t(x_t)||^2 dt \Bigg] \ \ \textrm{s.t.} \ \ P^v_1(x_1) = \pi_1(x_1)
\end{align}

Next we summarize two recent algorithms that solve the Schr\"{o}dinger bridge problem exactly (at least in theory): Iterative Proportional Filtering (IPF) and Iterative Markovian Fitting (IMF).

\subsection{Iterative Proportional Filtering (IPF)}\label{sec:ipf}

IPF aims to solve the Schr\"{o}dinger bridge problem (\ref{eq:sb}) by alternating the forward and reverse half bridge (HB) problems until convergence. More specifically, with initial $P^0 = P^{ref}$, we solve the followings for $n=1,2,\dots$ 
\begin{align}
&\textrm{(Reverse HB)} \ \ P^{2n-1} = \arg\min_{P} \ \textrm{KL}(P || P^{2n-2}) \ \ \textrm{s.t.} \ P_1(x_1) = \pi_1(x_1) \label{eq:ipf_rev} \\
&\textrm{(Forward HB)} \ \ P^{2n} = \arg\min_{P} \ \textrm{KL}(P || P^{2n-1}) \ \ \textrm{s.t.} \ P_0(x_0) = \pi_0(x_0) 
\label{eq:ipf_fwd}
\end{align}
It can be shown that $\lim_{n\to\infty} P^n \to P^{SB}$~
\citep{fortet40,kullback68,ruschendorf95}. 
It is not difficult to show that the optimal solution of (\ref{eq:ipf_rev}) or (\ref{eq:ipf_fwd}) can be attained by time-reversing the SDE of the previous iteration.
This fact was exploited recently in~\citep{bortoli21,vargas21} to yield neural-network based IPF algorithms where the score $\nabla \log P^n(x)$ that appears in time reversal is estimated either by regression estimation~\citep{bortoli21} or maximum likelihood estimation~\citep{vargas21}. 
However, the main drawback of these IPF algorithms is that they are simulation-based methods, thus very expensive to train.

\subsection{Iterative Markovian Fitting (IMF)}\label{sec:imf}

Recently in~\citep{dsbm}, the concept of path measure projection was introduced, specifically the Markovian and reciprocal projections that preserve the boundary marginals of the path measure. This idea was developed into a novel matching algorithm called the {\em iterative Markovian fitting} (IMF) that alternates applying the two projections starting from the initial path measure. Not only is it shown to converge to $P^{SB}$, but the algorithm is computationally more efficient than IPF without relying on simulation-based learning. A practical version of the algorithm is dubbed {\em Deep Schr\"{o}dinger Bridge Matching} (DSBM). 

We begin with discussing the two projections.

\textbf{$\bullet$ Reciprocal projection.} 
They define the {\em reciprocal class} of path measures to be the set of path measures that admit $P^{ref}(\cdot|x_0,x_1)$ as their pinned conditional path measures. That is, the {\em reciprocal class} $\mathcal{R}$ is defined as:
\begin{align}
\mathcal{R} = \{P: P(x_0,x_1) P^{ref}(\cdot|x_0,x_1) \}
\end{align}
The {\em reciprocal projection} of a path measure $P$, denoted by $\Pi_\mathcal{R}(P)$, is defined as the path measure in the reciprocal class that is closest to $P$ in the KL divergence sense. Formally, 
\begin{align}
\Pi_\mathcal{R}(P) = \arg\min_{R\in\mathcal{R}} \textrm{KL}(P||R) = P(x_0,x_1)P^{ref}(\cdot|x_0,x_1)
\end{align}
where the latter equality can be easily derived from the KL decomposition property. So it basically says that the reciprocal projection of $P$ is simply done by replacing $P(\cdot|x_0,x_1)$ by that of $P^{Ref}$ while keeping the coupling $P(x_0,x_1)$.

\textbf{$\bullet$ Markovian projection.} 
They also define the {\em Markovian class} as the set of any SDE-representable path measures with diffusion coefficient $\sigma_{ref}$. That is,
\begin{align}
\mathcal{M} = \{P: dx_t = g_t(x_t) dt + \sigma_{ref} dW_t \ \ \textrm{for any vector field} \ g \}
\end{align}
The {\em Markovian projection} of a path measure $P$, denoted by $\Pi_\mathcal{M}(P)$, is defined similarly as the path measure in the Markovian class that is closest to $P$ in the KL divergence sense,
\begin{align}
\Pi_\mathcal{M}(P) = \arg\min_{M\in\mathcal{M}} \textrm{KL}(P||M)
\end{align}
In~\citep{dsbm} (Proposition 2 therein), it was shown that $\Pi_\mathcal{M}(P)$ can be expressed succinctly for reciprocal path measures $P$. Specifically, for $P\in\mathcal{R}$, we have $\Pi_\mathcal{M}(P) = P^{v^*}$ where $P^{v^*}$ is described by the SDE: $dx_t = v^*_t(x_t) dt + \sigma_{ref} dW_t, \ x_0\!\sim\!P(x_0)$ where
\begin{align}
v^*_t(x_t) = \mathbb{E}_{P(x_1|x_t)} \Big[ \sigma_{ref}^2 \nabla_{x_t} \log P^{ref}(x_1|x_t) \Big] = \mathbb{E}_{P(x_1|x_t)} \bigg[ \frac{x_1-x_t}{1-t}\bigg]
\label{eq:markovian_proj_drift}
\end{align}
where the latter equality comes immediately from the closed-form $P^{ref}(x_1|x_t)=\mathcal{N}(x_t, \sigma_{ref}^2(1-t)I)$. 
Also it was shown that the marginals are preserved after the projection, that is, $P^{v^*}_t(\cdot) = P_t(\cdot)$ for all $t\in[0,1]$. This means that applying any number of Markovian (and also reciprocal) projections to a path measure $P$ always preserves the boundary marginals $P_0(x_0)$ and $P_1(x_1)$. And this is one of the key theoretical underpinnings of their algorithms called IMF and its practical version DSBM (details below) to solve the Schr\"{o}dinger bridge problem. 

\textbf{$\bullet$ IMF and DSBM algorithms.} 
Conceptually the IMF algorithm can be seen as a successive alternating application of the Markovian and reciprocal projections, starting from any initial path measure $P^0$ that satisfies $P^0(x_0)=\pi_0(x_0)$ and $P^0(x_1)=\pi_1(x_1)$ (e.g., $P^0 = \pi_0(x_0) \pi_1(x_1) P^{ref}(\cdot|x_0,x_1)$ is a typical choice). That is, for $n=1,2,\dots$
\begin{align}
P^{2n-1} = \Pi_{\mathcal{M}}(P^{2n-2}), \ \ \ \ P^{2n} = \Pi_{\mathcal{R}}(P^{2n-1})
\end{align}
Not only do all $\{P^n\}_{n\geq 0}$ meet the boundary conditions (i.e., $P^n_0=\pi_0$, $P^n_1=\pi_1$), 
it can be also shown that they keep getting closer to $P^{SB}$, and converge to $P^{SB}$ (i.e., $\textrm{KL}(P^{n+1}||P^{SB}) \leq \textrm{KL}(P^{n}||P^{SB})$ and $\lim_{n\to\infty} P^{n} = P^{SB}$)~\citep{dsbm} (Proposition 7 and Theorem 8 therein). 
The reciprocal projection is straightforward as it only requires sampling from the pinned process $P^{ref}(\cdot|x_0,x_1)$ that is done by running $dx_t = \frac{x_1-x_t}{1-t} dt + \sigma_{ref} dW_t$ with $(x_0,x_1)$ taken from the previous path measure. However, the Markovian projection involves the difficult $P(x_1|x_t)$ in (\ref{eq:markovian_proj_drift}) from the previous path measure $P$. To circumvent $P(x_1|x_t)$, they used the regression theorem by introducing a neural network $v_\theta(t,x)$ to approximate $v^*_t(x)$ and optimizing the following:
\begin{align}
\arg\min_\theta 
\int_0^1 \mathbb{E}_{P(x_t,x_1)} \big\vert\big\vert v_\theta(t,x_t) - \sigma_{ref}^2 \nabla_{x_t} \log P^{ref}(x_1|x_t) \big\vert\big\vert^2 dt
\label{eq:imf_nnet_optim}
\end{align}
where now the cached samples $(x_1,x_t)$ from the previous path measure $P$ can be used to solve (\ref{eq:imf_nnet_optim}). 
Although theoretically $v_{\theta^*}(t,x) = v^*_t(x)$ with ideally rich neural network capacity and perfect optimization, in practice due to the neural network approximation error, the boundary condition is not satisfied, i.e., $P^{2n-1}_1 \neq \pi_1$. Hence to mitigate the issue, they proposed IMF's practical version, called the Diffusion Schrodinger Bridge Matching (DSBM) algorithm~\citep{dsbm}. The idea is to do Markovian projections with both forward and reverse-time SDEs in an alternating fashion where the former starts from $\pi_0$ and the latter from $\pi_1$, which was shown to mitigate the boundary condition issue.

\section{A Unified Framework for Diffusion Bridge Matching Problems
}\label{sec:unified}

Our proposed unified framework is described in Alg.~\ref{alg:sde_cfm}. It can be seen as an extension of the CFM algorithm~\citep{cfm} where the only difference is that we consider the SDE bridge instead of the ODE bridge (i.e., the diffusion term in step 2). 
But this difference is crucial, as will be shown, allowing us to resolve the limitations of the CFM discussed in Sec.~\ref{sec:fm_cfm}. It also makes the framework general enough to subsume the IMF/DSBM algorithm for the Schr\"{o}dinger bridge problem and various ODE bridge algorithms as special cases. 
We also emphasize that even though this small change of adding the diffusion term in step 2 may look minor, its theoretical consequence, specifically our theoretical result in Theorem~\ref{thm:sde_cfm_general}, has rarely been studied in the literature by far.

We call the unified framework {\em Unified Bridge Algorithm} (UBA for short). Note that UBA described in Alg.~\ref{alg:sde_cfm} can deal with both ODE and SDE bridge problems, and if the diffusion coefficient $\sigma$ vanishes, it reduces to CFM for ODE bridge. Similarly as CFM, under the assumption of rich enough neural network functional capacity and perfect optimization solutions, our framework {\em guarantees} to solve the bridge problem. More formally, we have the following theorem.
\begin{theorem}[Our Unified Bridge Algorithm (UBA) solves the bridge problem]
If the neural network $v_\theta(t,x)$ functional space is rich enough to approximate any function arbitrarily closely, and if the optimization in step 3 can be solved perfectly, then each iteration of going through steps 1--3 in Alg.~\ref{alg:sde_cfm} ensures that $dx_t = v_\theta(t,x_t) dt + \sigma dW_t, \ x_0\!\sim\!\pi_0(\cdot)$ (after the optimization in step 3) admits $\{P_t(x_t)\}_t$ of (\ref{eq:sde_cfm_pt}) as its marginal distributions.
\label{thm:sde_cfm_general}
\end{theorem}

The proof can be found in Appendix~\ref{appsec:proofs}. 
The theorem says that after each iteration of going through steps 1--3, it is always guaranteed that the current SDE admits $\{P_t(x_t)\}_t$ defined in step 1 as marginal distributions. Since $P_0(\cdot) = \pi_0(\cdot)$ and $P_1(\cdot) = \pi_1(\cdot)$, the bridge problem is solved. Depending on the design choice, one can have just one iteration to solve the bridge problem. Under certain choices, however, it might be necessary to run the iterations many times to find the desired bridge solutions (e.g., mini-batch OT-CFM~\citep{cfm} and the IMF/DSBM Schr\"{o}dinger bridge matching algorithm~\citep{dsbm} as we illustrate in Sec.~\ref{sec:special_case_mot_cfm} and Sec.~\ref{sec:special_case_sb}, respectively). 

In the subsequent sections, we illustrate how several popular ODE bridge and Schr\"{o}dinger bridge algorithms can be instantiated as spcial cases of our UBA framework.

\newcommand\inlineeqno{\stepcounter{equation}\ (\theequation)}
\newcommand{\INDSTATE}[1][1]{\STATE\hspace{#1\algorithmicindent}}
\begin{algorithm}[t!]
\caption{Our Unified Bridge Algorithm (UBA) for bridge problems.
}
\label{alg:sde_cfm}
\begin{small}
\begin{algorithmic}
\STATE \textbf{Input:} The end-point distributions $\pi_0$ and $\pi_1$ (i.e., samples from them).
\STATE \textbf{Repeat} until convergence or a sufficient number of times:
    \INDSTATE[1] 1. Choose a pinned marginal path $\{P_t(x|x_0,x_1)\}_t$ and a coupling  distribution $Q(x_0,x_1)$ such that 
    \INDSTATE[1] \ \ \ \ $P_0(\cdot)=\pi_0(\cdot)$ and $P_1(\cdot)=\pi_1(\cdot)$ where
    \begin{align}
    P_t(x_t) := \int P_t(x_t|x_0,x_1) Q(x_0,x_1) d(x_0,x_1)
    \label{eq:sde_cfm_pt}
    \end{align}
    \INDSTATE[1] 2. Choose $\sigma \geq 0$, and find $u_t(x|x_0,x_1)$ such that the SDE 
    \begin{align}
    dx_t = u_t(x_t|x_0,x_1) dt + \sigma dW_t
    \label{eq:sde_cfm_sde}
    \end{align}
    \ \ \ \ \ \ \ \ \ \ admits $\{P_t(x|x_0,x_1)\}_t$ as its marginals. (Note: many possible choices for $\sigma$ and $u_t(x|x_0,x_1)$) 
    %
    \INDSTATE[1] 3. Solve the following optimization problem with respect to the neural network $v_\theta(t,x)$:
    \begin{align}
    \min_\theta \ \mathbb{E}_{t, Q(x_0,x_1) P_t(x_t|x_0,x_1)}||u_t(x_t|x_0,x_1)-v_\theta(t,x_t)||^2
    \label{eq:sde_cfm_optim}
    \end{align}
\STATE \textbf{Return:} The learned SDE $dx_t = v_\theta(t,x_t) dt + \sigma dW_t$ as the bridge problem solution.
\end{algorithmic}
\end{small}
\end{algorithm}

\subsection{A Special Case: (Mini-batch) Optimal Transport CFM~\citep{cfm}}\label{sec:special_case_mot_cfm}

Within our general Unified Bridge Algorithm (UBA) framework (Alg.~\ref{alg:sde_cfm}), we select $P_t(x|x_0,x_1)$, $Q(x_0,x_1)$ and $u_t(x|x_0,x_1)$ as follows. First in step 1, 
\begin{align}
P_t(x_t|x_0,x_1) = \mathcal{N}(x_t; (1-t)x_0 + tx_1, \sigma_{min}^2 I)
\end{align}
\begin{align}
Q(x_0,x_1) = P^{mOT}(x_0,x_1) := 
\sum_{i\in B_0} \sum_{j\in B_1} \delta(x_0=x_0^i) \delta(x_1=x_1^j) p^{mOT}_{ij}
\end{align}
where $\sigma_{min}\to 0$, and  $(\{x^i_0\}_{i\in B_0}, \{x^j_1\}_{j\in B_1})$ is the mini-batch data, and $\{p^{mOT}_{ij}\}_{ij}$ is a $(|B_0| \times |B_1|)$ mini-batch OT solution matrix learned with the mini-batch as training data. That is,
\begin{align}
p^{mOT} = \arg\min_{p} \sum_{i,j} p_{ij} ||x^i_0-x^j_1||^2 \ \ \ \ \textrm{s.t.} \ \ \sum_{j\in B_1} p_{ij} = \frac{1}{|B_0|}, \ \sum_{i\in B_0} p_{ij} = \frac{1}{|B_1|}
\end{align}
It is worth mentioning that $\sigma_{min}\to 0$ is required to have boundary consistency for $P_t(x_t|x_0,x_1)$ at $t=0$ and $t=1$. 
Note also that $P_t(x|x_0,x_1)$ is always fixed over iterations while $Q(x_0,x_1)$ varies over iterations depending on the mini-batch data sampled. Note that in our UBA framework, each iteration allows for different choices of $P_t(x|x_0,x_1)$ and $Q(x_0,x_1)$. 

In step 2, we choose $\sigma=0$, and define $u_t(x_t|x_0,x_1)$ to be a constant (independent on $t$) straight line vector from $x_0$ to $x_1$, i.e., 
\begin{align}
u_t(x_t|x_0,x_1) = x_1 - x_0
\end{align}
which can be shown to make the ODE $dx_t = u_t(x_t|x_0,x_1)dt$ admit $P_t(x_t|x_0,x_1)$ as its marginal distributions~\citep{cfm}.

The above choices precisely yield the (mini-batch) optimal transport CFM (OT-CFM) introduced in~\citep{cfm}. 


\subsection{A Special Case: (Mini-batch) Schr\"{o}dinger Bridge CFM~\citep{cfm}}\label{sec:special_case_msb_cfm}

Within our general Unified Bridge Algorithm (UBA) framework (Alg.~\ref{alg:sde_cfm}), we select $P_t(x|x_0,x_1)$, $Q(x_0,x_1)$ and $u_t(x|x_0,x_1)$ as follows. First in step 1, 
\begin{align}
P_t(x_t|x_0,x_1) = P^{ref}_t(x_t|x_0,x_1) = \mathcal{N}(x_t; (1-t)x_0 + tx_1, \sigma_{ref}^2 t(1-t) I)
\end{align}
\begin{align}
Q(x_0,x_1) = P^{mEOT}(x_0,x_1) := \sum_{i\in B_0} \sum_{j\in B_1} \delta(x_0=x_0^i) \delta(x_1=x_1^j) p^{mEOT}_{ij}
\label{eq:minibatch_q}
\end{align}
where $(\{x^i_0\}_{i\in B_0}, \{x^j_1\}_{j\in B_1})$ is the mini-batch data, and $\{p^{mEOT}_{ij}\}_{ij}$ is a $(|B_0| \times |B_1|)$ is the mini-batch entropic OT solution matrix with the negative entropy regularizing coefficient $2\sigma_{ref}^2$ (e.g., from the Sinkhorn-Knopp algorithm) learned with the mini-batch as training data. 
Although the samples from the coupling distribution $Q(x_0,x_1)$ in (\ref{eq:minibatch_q}) over the iterations in Alg.~\ref{alg:sde_cfm} conform to the data distributions $\pi_0$ and $\pi_1$ marginally, the mini-batch entropic OT solution is generally substantially different from the population entropic OT solution (the optimal solution of the Schr\"{o}dinger Bridge).

In step 2, we choose $\sigma=0$, and define $u_t(x_t|x_0,x_1)$ to be:
\begin{align}
u_t(x_t|x_0,x_1) = \frac{1-2t}{2t(1-t)} (x_t - (t x_1 + (1-t) x_0)) + x_1 - x_0
\end{align}
which can be shown to make the ODE $dx_t = u_t(x_t|x_0,x_1)dt$ admit $P_t(x_t|x_0,x_1)$ as its marginal distributions~\citep{cfm}.

The above choices precisely yield the (mini-batch) Schr\"{o}dinger Bridge CFM (SB-CFM) introduced in~\citep{cfm}. Although the marginals of SB-CFM match those of the Schr\"{o}dinger bridge solution, the entire path measure not since it only finds an ODE bridge solution.

\subsection{A Special Case: IMF or Deep Schr\"{o}dinger Bridge Matching (DSBM)~\citep{dsbm}}\label{sec:special_case_sb}

As discussed in~\ref{sec:imf}, the IMF/DSBM algorithm is based on the IMF principle where starting from $P(\{x_t\}_{t\in[0,1]}) = \pi_0(x_0) \pi_1(x_1) P^{ref}(\{x_t\}_{t\in(0,1)}|x_0,x_1)$, repeatedly and alternatively applying the projections $P \leftarrow \Pi_{\mathcal{M}}(P)$ and $P \leftarrow \Pi_{\mathcal{R}}(P)$ leads to convergence to the Schr\"{o}dinger bridge solution.
How does this algorithm fit in the framework of our Unified Bridge Algorithm (UBA) in Alg.~\ref{alg:sde_cfm}? We will see that a specific choice of $P_t(x|x_0,x_1)$, $Q(x_0,x_1)$ (in step 1) and $u_t(x|x_0,x_1)$ (in step 2) precisely leads to the IMF algorithm. We describe the algorithm in Alg.~\ref{alg:sde_cfm_imf}. 

In step 1, the pinned path marginals $P_t(x|x_0,x_1)$ are set to be equal to $P^{ref}_t(x|x_0,x_1)$ which can be written analytically as Gaussian (\ref{eq:sde_cfm_imf_1_p}). 
The coupling $Q(x_0,x_1)$ 
is defined to be the coupling distribution $P^{v_{\theta}}(x_0,x_1)$ that is induced from the SDE in the previous iteration (step 3), $P^{v_\theta}: dx_t = v_\theta(t,x_t) dt + \sigma dW_t$. In the first iteration where no $\theta$ is available yet, we set $Q(x_0,x_1):= \pi_0(x_0) \pi_1(x_1)$. We need to check if the boundary condition for (\ref{eq:sde_cfm_pt}) is satisfied.  This will be done shortly in the following paragraph. 
In step 2, we fix $\sigma:=\sigma_{ref}$, and set $u_t(x_t|x_0,x_1) := \sigma_{ref}^2 \nabla_{x_t} \log P^{ref}(x_1|x_t) = \frac{x_1-x_t}{1-t}$. In step 3 we update $\theta$ by solving the optimization  (\ref{eq:sde_cfm_imf_3}), the same as (\ref{eq:sde_cfm_optim}), with the chosen $P_t(x|x_0,x_1)$, $Q(x_0,x_1)$ and $u_t(x_t|x_0,x_1)$. 

Now we see how this choice leads to the IMF algorithm precisely. First, due to Doob's h-transform~\citep{rogers_williams}, the SDE $dx_t = u_t(x_t|x_0,x_1) dt + \sigma dW_t$ with the choice (\ref{eq:sde_cfm_imf_2}) admits $\{P^{ref}_t(x|x_0,x_1)\}_t$ as its marginals for any $(x_0,x_1)$. 
Next, the step 3, if optimized perfectly and ideally with zero neural net approximation error, is equivalent to $\Pi_{\mathcal{M}}(\Pi_{\mathcal{R}}(P^{v_{\theta_{old}}}))$ where $\theta_{old}$ is the optimized $\theta$ in the previous iteration\footnote{
Initially when there is no previous $\theta_{old}$ available, the step 3 is equivalent to $\Pi_{\mathcal{M}}(P^{init})$ where $P^{init}=\pi_0(x_0)\pi_1(x_1)P^{ref}(\cdot|x_0,x_1)$ which is already in the reciprocal class $\mathcal{R}$.
}. This can be easily understood by looking at the Markovian projection $\Pi_{\mathcal{M}}(P)$ written in the optimization form (\ref{eq:imf_nnet_optim}): The expectation is taken with respect to $P(x_t,x_1)$ that matches $Q(x_0,x_1) P^{ref}_t(x_t|x_0,x_1)$ in (\ref{eq:sde_cfm_imf_3}), and which is exactly the reciprocal projection of $P^{v_{\theta_{old}}}$ since $Q(x_0,x_1) = P^{v_{\theta_{old}}}(x_0,x_1)$ by construction. 
Lastly, we can verify that the choice in step 1 ensures the boundary conditions $P_0(\cdot)=\pi_0(\cdot)$, $P_1(\cdot)=\pi_1(\cdot)$. 
This is because $Q(x_0,x_1)$ always satisfies $Q(x_0) = \pi_0(x_0)$ and $Q(x_1) = \pi_0(x_1)$: initially  $Q(x_0,x_1)=\pi_0(x_0)\pi_1(x_1)$ obviously, and later as $Q(x_0,x_1)=P^{v_{\theta}}(x_0,x_1)$
results from the Markovian projection (from step 3 in the previous iteration). We recall from Sec.~\ref{sec:imf} that the Markovian projection preserves the boundary conditions. 

\begin{algorithm}[t!]
\caption{IMF algorithm~\citep{dsbm} as a special instance of our UBA.
}
\label{alg:sde_cfm_imf}
\begin{small}
\begin{algorithmic}
\STATE \textbf{Input:} The end-point distributions $\pi_0$ and $\pi_1$ (i.e., samples from them).
\STATE \textbf{Repeat} until convergence or a sufficient number of times:
    \INDSTATE[1] 1. Choose a pinned marginal path $\{P_t(x|x_0,x_1)\}_t$ and a coupling  distribution $Q(x_0,x_1)$ as follows:
    \begin{align}
    &P_t(x_t|x_0,x_1) := P^{ref}_t(x_t|x_0,x_1) = \mathcal{N}(x_t; (1-t) x_0 + t x_1, \sigma_{ref}^2 t(1-t) I) \label{eq:sde_cfm_imf_1_p} \\ 
    &Q(x_0,x_1) := \begin{cases}
        \pi_0(x_0) \pi_1(x_1) & \text{initially (if $\theta$ is not available)} \\
        P^{v_{\theta}}(x_0,x_1) & \text{otherwise}
    \end{cases}
    \label{eq:sde_cfm_imf_1_q}
    \end{align}
    \INDSTATE[1] 2. Choose $\sigma:=\sigma_{ref}$, and set $u_t(x|x_0,x_1)$ as: 
    \begin{align}
    u_t(x_t|x_0,x_1) := \sigma_{ref}^2 \nabla_{x_t} \log P^{ref}(x_1|x_t) = \frac{x_1-x_t}{1-t}
    \label{eq:sde_cfm_imf_2}
    \end{align}
    %
    \INDSTATE[1] 3. Solve the following optimization problem with respect to the neural network $v_\theta(t,x)$:
    \begin{align}
    \min_\theta \ \mathbb{E}_{t, Q(x_0,x_1) P_t(x_t|x_0,x_1)}||u_t(x_t|x_0,x_1)-v_\theta(t,x_t)||^2
    \label{eq:sde_cfm_imf_3}
    \end{align}
\STATE \textbf{Return:} The learned SDE $dx_t = v_\theta(t,x_t) dt + \sigma dW_t$ as the bridge problem solution.
\end{algorithmic}
\end{small}
\end{algorithm}


\textbf{$\bullet$ IMF algorithm as a UBA in minimal kinetic energy forms.} 
We can reformulate the IMF algorithm within our UBA framework using the minimal kinetic form as described in Alg.~\ref{alg:sde_cfm_imf_min_kinetic}. 
In fact it can be shown that Alg.~\ref{alg:sde_cfm_imf} and Alg.~\ref{alg:sde_cfm_imf_min_kinetic} are indeed equivalent, as stated in Theorem~\ref{thm:min_kinetic} in Appendix~\ref{appsec:min_kinetic}. 
Our proof in Appendix~\ref{appsec:min_kinetic} relies on some results from the stochastic optimal control theory~\citep{raginsky19}. 
Then what is the benefit of having this stochastic optimal control formulation for the IMF algorithm? Compared to Alg~\ref{alg:sde_cfm_imf}, it has more flexibility allowing us to extend or re-purpose the bridge matching algorithm for different goals. For instance, the Generalized Schr\"{o}dinger Bridge Matching (GSBM)~\citep{gsbm} adopted a formulation similar to Alg.~\ref{alg:sde_cfm_imf_min_kinetic}, in which they introduced the stage cost function that is minimized together with the control norm term. The final solution SDE would {\em not} be the Schr\"{o}dinger bridge solution, but can be seen as a {\em generalized} solution that takes into account problem-specific stage costs. Hence the algorithmic framework in Alg.~\ref{alg:sde_cfm_imf_min_kinetic} is especially beneficial for developing new problem setups and novel bridge algorithms, encouraging researchers to explore further in this area for future research. 





\begin{algorithm}[t!]
\caption{IMF algorithm~\citep{dsbm} as a minimal kinetic energy form in our UBA.
}
\label{alg:sde_cfm_imf_min_kinetic}
\begin{small}
\begin{algorithmic}
\STATE \textbf{Input:} The end-point distributions $\pi_0$ and $\pi_1$ (i.e., samples from them).
\STATE \textbf{Repeat} until convergence or a sufficient number of times:
    \INDSTATE[1] 1. Choose $P_t(x|x_0,x_1) = \mathcal{N}(x; \mu_t, \gamma_t^2 I)$ where $(\mu_t, \gamma_t)$ are solutions to the 
    following optimization:
    \begin{align}
    &\arg\min_{\{\mu_t,\gamma_t\}_t} \int_0^1 \mathbb{E}_{P_t(x|x_0,x_1)} \bigg[ \frac{1}{2} || \alpha_t(x|x_0,x_1) ||^2 \bigg] dt \ \ \ \ \textrm{where} \label{eq:sde_cfm_imf_min_kinetic_1_p_1} \\
    & \ \ \ \ \alpha_t(x|x_0,x_1) = \frac{d\mu_t}{dt} + a_t(x-\mu_t), \ \ a_t = \frac{1}{\gamma_t} \bigg(\frac{d\gamma_t}{dt} - \frac{\sigma_{ref}^2}{2\gamma_t}\bigg)
    \label{eq:sde_cfm_imf_min_kinetic_1_p_2}
    \end{align}
    \ \ \ \ \ \ \ \ \ \ Choose a coupling  distribution $Q(x_0,x_1)$ as follows:
    \begin{align}
    Q(x_0,x_1) := \begin{cases}
        \pi_0(x_0) \pi_1(x_1) & \text{initially (if $\theta$ is not available)} \\
        P^{v_{\theta}}(x_0,x_1) & \text{otherwise}
    \end{cases}
    \label{eq:sde_cfm_imf_min_kinetic_1_q}
    \end{align}
    \INDSTATE[1] 2. Choose $\sigma:=\sigma_{ref}$, and set $u_t(x|x_0,x_1):=\alpha_t(x|x_0,x_1)$.
    \INDSTATE[1] 3. Solve the following optimization problem with respect to the neural network $v_\theta(t,x)$:
    \begin{align}
    \min_\theta \mathbb{E}_{t, Q(x_0,x_1) P_t(x_t|x_0,x_1)}||u_t(x_t|x_0,x_1)-v_\theta(t,x_t)||^2
    \label{eq:sde_cfm_imf_min_kinetic_3}
    \end{align}
\STATE \textbf{Return:} The learned SDE $dx_t = v_\theta(t,x_t) dt + \sigma dW_t$ as the bridge problem solution.
\end{algorithmic}
\end{small}
\end{algorithm}


\section{Conclusion}\label{sec:conclusion}

In this article we have proposed a novel unified framework for the bridge problem, dubbed {\em Unified Bridge Algorithm} (UBA). We have shown that the UBA framework is general and flexible enough to subsume many existing (conditional) flow matching algorithms and iterative Schr\"{o}dinger bridge matching algorithms. The correctness of the UBA framework, i.e., that the UBA guarantees to meet the bridge boundary conditions in each iteration, has been proven rigorously under the universal approximation assumption for neural networks. 
In particular, we have illustrated how the existing Flow Matching (FM) algorithm, the (mini-batch) optimal transport FM algorithm, the (mini-batch) Schr\"{o}dinger bridge FM algorithm, and the deep Schr\"{o}dinger bridge matching (DSBM) algorithm can be instantiated as special cases within our UBA framework. Furthermore, our UBA framework with minimal kinetic energy forms can endow even more flexibility to allow for extending or re-purposing bridge matching algorithm for different goals. We believe that this unified framework will be useful for viewing the bridge problems in a more general and flexible perspective, and in turn can help researchers and practitioners to develop new bridge algorithms in their fields.



{
\bibliographystyle{spbasic}      
\bibliography{main}   
}

\clearpage

\appendix

\centerline{\huge\textbf{{Appendix}}}


\section{Theorems and Proofs}\label{appsec:proofs}

\subsection{Proof of Theorem~\ref{thm:sde_cfm_general}.}

To prove Theorem~\ref{thm:sde_cfm_general}, we show the following three lemmas in turn: 
\begin{enumerate}
\item (Lemma~\ref{lemma:1}) We first show that after step 2 of Alg.~\ref{alg:sde_cfm} is done, the SDE $dx_t = u_t(x_t) dt + \sigma dW_t, \ x_0\!\sim\!\pi_0(\cdot)$ admits $\{P_t(x_t)\}_t$ as its marginal distributions where $u_t(x) = \frac{1}{P_t(x)}\mathbb{E}_Q[u_t(x|x_0,x_1) P_t(x|x_0,x_1)]$.
\item Under the assumptions made in the theorem, that is, i) the neural network $v_\theta(t,x)$'s functional space is rich enough to approximate any function arbitrarily closely; ii) the step 3 of Alg.~\ref{alg:sde_cfm} is solved perfectly, we will show that the solution to step 3 is $v_\theta(t,x) = \mathbb{E}[u_t(x|x_0,x_1)|x_t\!=\!x]$. This straightforwardly comes from the regression theorem, but we will elaborate it in greater detail in Lemma~\ref{lemma:2} below. We then show that this conditional expectation equals $u_t(x)$ defined in (\ref{eq:utx}), i.e., $v_\theta(t,x) = u_t(x)$. This will complete the proof, and we assert that $dx_t = v_{\theta^*}(t, x_t) dt + \sigma_{ref} dW_t, \ x_0\!\sim\!\pi_0(\cdot)$ admits $\{P_t(x_t)\}_t$ as its marginals.
\item Practically, the training dynamics of the gradient descent for step 3 of Alg.~\ref{alg:sde_cfm} can be shown to be identical to that of minimizing $\mathbb{E}||v_\theta(t,x)-u_t(x)||^2$. This is done in Lemma~\ref{lemma:3} although the proof is very similar to the result in~\cite{cfm}. Hence, in practice, even without the assumptions of the ideal rich neural network functional capacity and perfect optimization, we can continue to reduce the error between $v_\theta(t,x)$ and $u_t(x)$ in the course of gradient descent for step 3.
\end{enumerate}

\begin{lemma}
Suppose $\{P_t(x|x_0,x_1)\}_t$ be the marginal distributions of the SDE $dx_t = u_t(x_t|x_0,x_1)dt + \sigma dW_t$ for given $x_0$ and $x_1$. In other words, step 2 of Alg.~\ref{alg:sde_cfm} is done. For
\begin{align}
P_t(x) &:= \int P_t(x|x_0,x_1) Q(x_0,x_1) d(x_0,x_1) \label{eq:ptx} \\
u_t(x) &:= \frac{1}{P_t(x)}\mathbb{E}_{Q(x_0,x_1)}[u_t(x|x_0,x_1) P_t(x|x_0,x_1)], \label{eq:utx}
\end{align}
the SDE $dx_t=u_t(x_t)dt + \sigma dW_t$, $x_0\!\sim\!\pi_0(\cdot)$ has marginal distributions $\{P_t(x)\}_t$.
\label{lemma:1}
\end{lemma}
\begin{proof}
For the given $x_0$ and $x_1$, we apply the Fokker-Planck equation to the SDE $dx_t = u_t(x_t|x_0,x_1)dt + \sigma dW_t$ with its marginals $\{P_t(x|x_0,x_1)\}_t$.
\begin{align}
\frac{\partial}{\partial t} P_t(x|x_0,x_1) = -\Div \{P_t(x|x_0,x_1) u_t(x|x_0,x_1)\} + \frac{\sigma^2}{2} \Delta P_t(x|x_0,x_1)
\label{eq:fokker_planck}
\end{align}
where $\Div$ is the divergence operator and $\Delta$ is the Laplace operator. 
Now we derive the Fokker-Planck equation for the target SDE as follows:
\begin{align}
&\frac{\partial}{\partial t} P_t(x) = \frac{\partial}{\partial t} \int P_t(x|x_0,x_1) Q(x_0,x_1) d(x_0,x_1) \label{eq:fp_deriv_1} \\ 
&\ \ \ \ \ \ = \int \frac{\partial}{\partial t} P_t(x|x_0,x_1) Q(x_0,x_1) d(x_0,x_1) \label{eq:fp_deriv_2} \\ 
&\ \ \ \ \ \ = \int \Big( -\Div \{P_t(x|x_0,x_1) u_t(x|x_0,x_1)\} + \frac{\sigma^2}{2} \Delta P_t(x|x_0,x_1) \Big) Q(x_0,x_1) d(x_0,x_1) \label{eq:fp_deriv_3} \\
&\ \ \ \ \ \ = -\Div \mathbb{E}_{Q}\big[u_t(x|x_0,x_1) P_t(x|x_0,x_1)\big] + \frac{\sigma^2}{2} \Delta \int P_t(x|x_0,x_1) Q(x_0,x_1) d(x_0,x_1) \label{eq:fp_deriv_4} \\
&\ \ \ \ \ \ = -\Div \bigg\{ P_t(x) \frac{1}{P_t(x) }\mathbb{E}_{Q}\big[u_t(x|x_0,x_1) P_t(x|x_0,x_1)\big] \bigg\} + \frac{\sigma^2}{2} \Delta P_t(x) \label{eq:fp_deriv_5} \\
&\ \ \ \ \ \ = -\Div \{P_t(x) u_t(x)\} + \frac{\sigma^2}{2} \Delta P_t(x) \label{eq:fp_deriv_5}
\end{align}
This establishes the Fokker-Planck equation for the SDE $dx_t=u_t(x_t)dt + \sigma dW_t$, $x_0\!\sim\!\pi_0(\cdot)$, to which $\{P_t(x)\}_t$ is the solution. This completes the proof of Lemma~\ref{lemma:1}.  
\end{proof}

\begin{lemma}
Under the assumptions of ideal rich neural network capacity and perfect optimization made in the theorem, the solution to step 3 of Alg.~\ref{alg:sde_cfm}, that is, 
\begin{align}
\theta^* = \arg\min_\theta \ \mathbb{E}_{t, Q(x_0,x_1) P_t(x_t|x_0,x_1)}||u_t(x_t|x_0,x_1)-v_\theta(t,x_t)||^2
\label{eq:sde_cfm_optim_re}
\end{align}
satisfies $v_{\theta^*}(t,x) = \mathbb{E}[u_t(x|x_0,x_1)|x_t\!=\!x] = u_t(x)$ where $u_t(x)$ is defined in (\ref{eq:utx}). 
\label{lemma:2}
\end{lemma}
\begin{proof}
We prove the second equality first. The expectation $\mathbb{E}[u_t(x_t|x_0,x_1)|x_t]$ is taken with respect to the distribution $R(x_0,x_1|x_t)$ defined to be proportional to $P_t(x_t|x_0,x_1) Q(x_0,x_1)$. 
\begin{align}
&\mathbb{E}[u_t(x_t|x_0,x_1)|x_t] = \int R(x_0,x_1|x_t) u_t(x_t|x_0,x_1) d(x_0,x_1) \label{condexp_1} \\
& \ \ \ \ \ \ = \int \frac{P_t(x_t|x_0,x_1) Q(x_0,x_1)}{\int P_t(x_t|x_0,x_1) Q(x_0,x_1) d(x_0,x_1)} u_t(x_t|x_0,x_1) d(x_0,x_1) \label{condexp_2} \\
& \ \ \ \ \ \ = \int \frac{P_t(x_t|x_0,x_1) Q(x_0,x_1)}{P_t(x_t)} u_t(x_t|x_0,x_1) d(x_0,x_1) \label{condexp_3} \\
& \ \ \ \ \ \ = \frac{1}{P_t(x_t)} \int P_t(x_t|x_0,x_1) Q(x_0,x_1) u_t(x_t|x_0,x_1) d(x_0,x_1) \label{condexp_4} \\
& \ \ \ \ \ \ = \frac{1}{P_t(x_t)} \mathbb{E}_{Q(x_0,x_1)}\big[u_t(x_t|x_0,x_1) P_t(x_t|x_0,x_1)\big] \label{condexp_5} \\
& \ \ \ \ \ \ = u_t(x_t)
\end{align}
We now prove the first equality. Although this straightforwardly comes from the regression theorem, but here we will elaborate it in greater detail. 
Due to the assumptions, the optimization (\ref{eq:sde_cfm_optim_re}) can be written in a functional form as:
\begin{align}
v^* = \arg\min_{v(\cdot,\cdot)} \ \mathbb{E}_{t, Q(x_0,x_1) P_t(x_t|x_0,x_1)}||u_t(x_t|x_0,x_1)-v(t,x_t)||^2
\label{eq:sde_cfm_optim_func}
\end{align}
where its optimizer $v^*(t,x)$ equals the optimizer $v_{\theta^*}(t,x)$ of (\ref{eq:sde_cfm_optim_re}). 
In the functional optimization (\ref{eq:sde_cfm_optim_func}), the objective is completely decomposed over $t$, and we can equivalently  minimize $\mathbb{E}_{Q(x_0,x_1) P_t(x_t|x_0,x_1)}||u_t(x_t|x_0,x_1)-v(t,x_t)||^2$ for each $t$. We take the functional gradient with respect to $v(t,\cdot)$. For ease of exposition, we will use simpler notation where we minimize $\mathbb{E}_{p(y,z)}||f(y,z)-g(y)||^2$ with respect to the function $g(\cdot)$. Hence there is direct correspondence: $y$ is to $x_t$, $z$ to $(x_0,x_1)$, $f(y,z)$ to $u_t(x_t|x_0,x_1)$, and $g$ to $v$. We see that at optimum, 
\begin{align}
\partial g(y) = \int 2 p(y,z) (g(y)-f(y,z)) dz = 0
\end{align}
leading to $g^*(y) = \mathbb{E}[f(y,z)|y]$. 
This regression theorem implies $v^*(t,x_t) = \mathbb{E}[u_t(x_t|x_0,x_1)|x_t]$.
This completes the proof Leamma~\ref{lemma:2}.
\end{proof}

\begin{lemma}
$\nabla_\theta \mathbb{E}_{P_t(x)}||v_\theta(t,x)-u_t(x)||^2 = \nabla_\theta \mathbb{E}_{P_t(x|x_0,x_1)Q(x_0,x_1)}||v_\theta(t,x)-u_t(x|x_0,x_1)||^2$ for each $t$, where $P_t(x)$ and $u_t(x)$ are defined as (\ref{eq:ptx}) and (\ref{eq:utx}), respectively.
\label{lemma:3}
\end{lemma}
\begin{proof}
\begin{align}
\nabla_\theta \mathbb{E}_{P_t(x)}||v_\theta(t,x)-u_t(x)||^2 &= \nabla_\theta \mathbb{E}_{P_t(x)}\Big[ ||v_\theta(t,x)||^2 - 2v_\theta(t,x)^\top u_t(x) \Big] \label{eq:grad_objs_1} \\ 
&= \nabla_\theta \mathbb{E}_{P_t(x)}||v_\theta(t,x)||^2 - 2 \nabla_\theta \mathbb{E}_{P_t(x)} \Big[ v_\theta(t,x)^\top u_t(x) \Big] \label{eq:grad_objs_2}
\end{align}
The first term can be written as:
\begin{align}
\nabla_\theta \mathbb{E}_{P_t(x)}||v_\theta(t,x)||^2 &= \nabla_\theta \int ||v_\theta(t,x)||^2 P_t(x) dx \label{eq:grad_objs_3} \\
&= \nabla_\theta \int ||v_\theta(t,x)||^2 P_t(x|x_0,x_1) Q(x_0,x_1) d(x,x_0,x_1) \label{eq:grad_objs_4} \\
&= \nabla_\theta \mathbb{E}_{P_t(x|x_0,x_1)Q(x_0,x_1)}||v_\theta(t,x)||^2 \label{eq:grad_objs_5}
\end{align}
The second term can be derived as:
\begin{align}
&\nabla_\theta \mathbb{E}_{P_t(x)} \Big[ v_\theta(t,x)^\top u_t(x) \Big] = \nabla_\theta \int v_\theta(t,x)^\top u_t(x) P_t(x) dx \label{eq:grad_objs_6} \\
&\ \ \ \ \ \ = \nabla_\theta \int v_\theta(t,x)^\top \mathbb{E}_{Q(x_0,x_1)}[u_t(x|x_0,x_1) P_t(x|x_0,x_1)] dx \label{eq:grad_objs_7} \\
&\ \ \ \ \ \ = \nabla_\theta \int v_\theta(t,x)^\top u_t(x|x_0,x_1) P_t(x|x_0,x_1) Q(x_0,x_1) d(x,x_0,x_1) \label{eq:grad_objs_8} \\
&\ \ \ \ \ \ = \nabla_\theta \mathbb{E}_{P_t(x|x_0,x_1)Q(x_0,x_1)}\Big[ v_\theta(t,x)^\top u_t(x|x_0,x_1) \Big] \label{eq:grad_objs_8}
\end{align}
Combining (\ref{eq:grad_objs_5}) and (\ref{eq:grad_objs_8}), and noting that $||u_t(x|x_0,x_1)||^2$ is independent of $\theta$, we complete the proof of Lemma~\ref{lemma:3}. 
\end{proof}

\subsection{Equivalence between Alg.~\ref{alg:sde_cfm_imf} and Alg.~\ref{alg:sde_cfm_imf_min_kinetic} }\label{appsec:min_kinetic}

\begin{theorem}
Under the same assumptions as those in Theorem~\ref{thm:sde_cfm_general}, Alg.~\ref{alg:sde_cfm_imf} and Alg.~\ref{alg:sde_cfm_imf_min_kinetic} lead to the same SDE solution, which is the Schr\"{o}dinger bridge matching. 
\label{thm:min_kinetic}
\end{theorem}
\begin{proof}
The proof goes as follows: i) We first show that the optimization problem (\ref{eq:sde_cfm_imf_min_kinetic_1_p_1}--\ref{eq:sde_cfm_imf_min_kinetic_1_p_2}) in step 1 in Alg.~\ref{alg:sde_cfm_imf_min_kinetic} can be seen as a constrained minimal kinetic energy optimal control problem with the constraint of Gaussian marginals $\{P_t(x|x_0,x_1)\}_t$; ii) We then relax it to an unconstrained version, and view the unconstrained problem as an instance of stochastic optimal control problem with fixed initial; iii) The latter is then shown to admit Gaussian pinned marginal solutions following the theory developed in~\citep{raginsky19}, thus proving that Gaussian constraining does not essentially restrict the problem. It turns out that the optimal pinned marginals and the optimal control have exactly the linear interpolation forms in (\ref{eq:sde_cfm_imf_1_p}) and (\ref{eq:sde_cfm_imf_2}), respectively, which completes the proof. 

In step 2, since we set $u_t(x|x_0,x_1) = \alpha_t(x|x_0,x_1)$ where $\alpha$ is the solution to (\ref{eq:sde_cfm_imf_min_kinetic_1_p_1}--\ref{eq:sde_cfm_imf_min_kinetic_1_p_2}), we will use the notation $u$ in place of $\alpha$ throughout the proof. First, we show that the SDE $dx_t = u_t(x_t|x_0,x_1) dt + \sigma dW_t$ with initial state $x_0$ at $t\!=\!0$ and $u$ satisfying (\ref{eq:sde_cfm_imf_min_kinetic_1_p_2}), admits $\{P_t(x_t|x_0,x_1)=\mathcal{N}(x_t; \mu_t,\gamma_t^2 I)\}_t$ as marginal distributions. Note that we must have $\mu_0=x_0$, $\mu_1=x_1$, $\gamma_0\to 0$, and $\gamma_1\to 0$ due to the conditioning (pinned process). This fact is in fact an extension of the similar one for ODE cases in~\citep{cfm}. We will do the proof here for SDE cases. We will establish the Fokker-Planck equation for the SDE, and we derive: 
\begin{align}
\frac{\partial P_t(x|x_0,x_1)}{\partial t} &= 
P_t(x|x_0,x_1) \cdot \frac{\partial \log \mathcal{N}(x; \mu_t, \gamma_t^2 I)}{\partial t} \\
&= P_t(x|x_0,x_1) \cdot \Bigg( - \frac{\gamma_t'}{\gamma_t}d + \frac{(x-\mu_t)^\top \mu_t'}{\gamma_t^2} + \frac{||x-\mu_t||^2\gamma_t'}{\gamma_t^3} \Bigg)
\label{eq:p9_lhs}
\end{align}
where $\mu_t'$ and $\gamma_t'$ are the time derivatives. 
We also derive the divergence and Laplacian as follows:
\begin{align}
&\Div\{ P_t(x|x_0,x_1) u_t(x|x_0,x_1) \} = \nonumber \\ & \ \ \ \ \ \ \ \ \ \ \ \ -P_t(x|x_0,x_1) \cdot \Bigg( -\frac{\gamma_t'}{\gamma_t}d + \frac{(x-\mu_t)^\top \mu_t'}{\gamma_t^2} + \frac{||x-\mu_t||^2 \Big( \gamma_t' - \frac{\sigma^2}{2\gamma_t} \Big)}{\gamma_t^3} + \frac{\sigma^2}{2\gamma_t^2} d \Bigg) 
\label{eq:p9_rhs1} \\
&\Delta P_t(x|x_0,x_1) = \textrm{Tr} \Big( \nabla_x^2 P_t(x|x_0,x_1) \Big) = P_t(x|x_0,x_1) \cdot \Bigg( \frac{||x-\mu_t||^2}{\gamma_t^4} - \frac{d}{\gamma_t^2} \Bigg) \label{eq:p9_rhs2}
\end{align}
From (\ref{eq:p9_lhs}), (\ref{eq:p9_rhs1}) and (\ref{eq:p9_rhs2}), we can establish the following equality, and it proves the fact. 
\begin{align}
\frac{\partial P_t(x|x_0,x_1)}{\partial t} = -\Div\{ P_t(x|x_0,x_1) u_t(x|x_0,x_1) \} + \frac{\sigma^2}{2} \Delta P_t(x|x_0,x_1)
\end{align}

From the above fact, we can re-state the step 1 of Alg.~\ref{alg:sde_cfm_imf_min_kinetic} as follows:

{\em (Step 1 re-stated) Choose $P_t(x|x_0,x_1)$ as the marginals of the SDE, $dx_t = u_t(x_t|x_0,x_1) dt + \sigma dW_t$ with initial state $x_0$ at $t\!=\!0$ where $u$ is the solution to the constrained optimization:}
\begin{align}
\min_{u} \int_0^1 \mathbb{E}_{P_t(x|x_0,x_1)} \bigg[ \frac{1}{2} || u_t(x|x_0,x_1) ||^2 \bigg] dt \ \  \textrm{s.t.} \ \ \{P_t(x|x_0,x_1)\}_t \ \ \textrm{are Gaussians}
\label{eq:constrained_min_kinetic}
\end{align}
Instead of solving (\ref{eq:constrained_min_kinetic}) directly, we try to deal with its unconstrained version, i.e., without the Gaussian marginal constraint. 
To this end we utilize the theory of stochastic optimal control with the fixed initial state~\citep{raginsky19}, which we adapted for our purpose below in Lemma~\ref{lemma:raginsky}. 

In Lemma~\ref{lemma:raginsky}, we adopt the Dirac's delta function $g(\cdot) = \delta_{x_1}$ for the terminal cost to ensure that the SDE $dx_t = u_t(x_t|x_0,x_1) dt + \sigma dW_t$ with initial state $x_0$ lands at $x_1$ as the final state. Then the unconstrained version of (\ref{eq:constrained_min_kinetic}), which is perfectly framed as an optimal control problem in Lemma~\ref{lemma:raginsky}, has the optimal solution written as:
\begin{align}
u_t(x|x_0,x_1) = \sigma^2 \nabla_x \log \mathbb{E}_{P^{ref}}[\delta_{x_1}|x_t=x] = \sigma^2 \nabla_x \log P^{ref}(x_1|x_t=x)
\label{eq:optimal_u}
\end{align}
which coincides with (\ref{eq:sde_cfm_imf_2}) in Alg.~\ref{alg:sde_cfm_imf}. Now, due to Doob's h-transform~\citep{rogers_williams}, the SDE $dx_t = u_t(x_t|x_0,x_1) dt + \sigma dW_t$ with the choice (\ref{eq:sde_cfm_imf_2}) or (\ref{eq:optimal_u}) admits $\{P^{ref}_t(x|x_0,x_1)\}_t$ as its marginals. In other words, $P_t(x|x_0,x_1) = P^{ref}_t(x|x_0,x_1)$, which is Gaussian, meaning that the constrained optimization (\ref{eq:constrained_min_kinetic}) and its unconstrained version essentially solve the same problem. 

Noting that (\ref{eq:sde_cfm_imf_1_p}) and (\ref{eq:sde_cfm_imf_2}) in Alg.~\ref{alg:sde_cfm_imf} are equivalent to (\ref{eq:sde_cfm_imf_min_kinetic_1_p_1}) and (\ref{eq:sde_cfm_imf_min_kinetic_1_p_2}) in Alg.~\ref{alg:sde_cfm_imf_min_kinetic}, we conclude that Alg.~\ref{alg:sde_cfm_imf} and Alg.~\ref{alg:sde_cfm_imf_min_kinetic} are equivalent.
\end{proof}

\begin{lemma}[Stochastic optimal control with fixed initial state; Adapted from Theorem 2.1 in~\citep{raginsky19}]
Let $P^b$ be the path measure of the SDE: $dx_t = b_t(x)dt + \sigma dW_t$, starting from the fixed initial state $x_0$. For the stochastic optimal control problem with the immediate cost $\frac{1}{2\sigma^2}||b_t(x_t)||^2$ at time $t$ and the terminal cost $\log 1/g(x_1)$ at final time $t=1$ for any function $g$, the cost-to-go function defined as:
\begin{align}
J_t^b(x) := \mathbb{E}_{P^b}\Bigg[ \int_t^1 \frac{1}{2\sigma^2}||b_t(x_t)||^2 -\log g(x_1) \bigg| x_t = x \Bigg]
\end{align}
has the optimal control (i.e., the optimal drift $b_t(x)$)
\begin{align}
b^*_t(x) = \arg\min_b J^b_t(x) = \sigma^2 \nabla_x \log \mathbb{E}_{P^{ref}}[g(x_1)|x_t=x]
\end{align}
where $P^{ref}$ is the Brownian path measure with diffusion coefficient $\sigma$. 
\label{lemma:raginsky}
\end{lemma}
\begin{proof}
We utilize the (simplified) Feynman–Kac formula, saying that the PDE,
\begin{align}
\frac{\partial h_t(x)}{\partial t} + \mu_t(x)^\top \nabla_x h_t(x) + \frac{1}{2}\textrm{Tr}\Big( \sigma^2 \nabla_x^2 h_t(x) \Big) = 0, \ \ h_1(\cdot) = q(\cdot)
\end{align}
has a solution $h_t(x) = \mathbb{E}[q(x_1)|x_t=x]$ where the expectation is taken with respect to the SDE, $dx_t =\mu_t(x_t)dt + \sigma dW_t$. 

Now we plug in $\mu_t=0$, $q = g$, and let $v_t(x) := -\log h_t(x)$. Note that $h_t(x)$ is always positive since $g$ is positive, and hence $v_t(x)$ is well defined. Then by some algebra, we see the following PDE:
\begin{align}
\frac{\partial v_t(x)}{\partial t} + \frac{1}{2}\textrm{Tr}\Big( \sigma^2 \nabla_x^2 v_t(x) \Big) = \frac{\sigma^2}{2} ||\nabla_x v_t(x)||^2, \ \ v_1(\cdot) = -\log g(\cdot)
\end{align}
has a solution  $v_t(x) = -\log \mathbb{E}_{P^{ref}}[g(x_1)|x_t=x]$. 
Note that we can write $\frac{\sigma^2}{2} ||\nabla_x v_t(x)||^2$ as the following variational form,
\begin{align}
\frac{\sigma^2}{2} ||\nabla_x v_t(x)||^2 = -\min_{b} \ b^\top \nabla_x v_t(x) + \frac{||b||^2}{2\sigma^2}
\label{eq:quad_optim}
\end{align}
where the minimum is attained at $b^* = -\sigma^2 \nabla_x v_t(x)$. So $v_t(x)=-\log \mathbb{E}_{P^{ref}}[g(x_1)|x_t=x]$ is the solution to:
\begin{align}
\frac{\partial v_t(x)}{\partial t} + \frac{1}{2}\textrm{Tr}\Big( \sigma^2 \nabla_x^2 v_t(x) \Big) = -\min_{b} \ b^\top \nabla_x v_t(x) + \frac{||b||^2}{2\sigma^2}, \ \ v_1(\cdot) = -\log g(\cdot)
\label{eq:hjb}
\end{align}
Note that (\ref{eq:hjb}) is the Hamilton-Jacobi-Bellman equation for the stochastic optimal control problem with the immediate cost $\frac{1}{2\sigma^2}||b_t(x_t)||^2$ and the terminal cost $\log 1/g(x_1)$. 
In fact, $v_t(x)=-\log \mathbb{E}_{P^{ref}}[g(x_1)|x_t=x]$ is the (optimal) value function, and the optimal control, i.e., the solution to (\ref{eq:quad_optim}) in a function form, is: $b^*_t(x) = -\sigma^2 \nabla_x v_t(x) = \sigma^2 \nabla_x \log \mathbb{E}_{P^{ref}}[g(x_1)|x_t=x]$.
\end{proof}

\end{document}